%% file: main.tex
\documentclass[sigconf]{acmart}
\AtBeginDocument{%
  \providecommand\BibTeX{{%
    \normalfont B\kern-0.5em{\scshape i\kern-0.25em b}\kern-0.8em\TeX}}}

\copyrightyear{2021}
\acmYear{2021}
\setcopyright{acmcopyright}\acmConference[CIKM '21]{Proceedings of the 30th ACM
International Conference on Information and Knowledge Management}{November
1--5, 2021}{Virtual Event, QLD, Australia}
\acmBooktitle{Proceedings of the 30th ACM International Conference on Information
and Knowledge Management (CIKM '21), November 1--5, 2021, Virtual Event, QLD,
Australia}
\acmPrice{15.00}
\acmDOI{10.1145/3459637.3482336}
\acmISBN{978-1-4503-8446-9/21/11}
\input{999_Notations}

\begin{document}
\fancyhead{}

\title{Efficient Hyperparameter Optimization\\ under Multi-Source Covariate Shift}

\author{Masahiro Nomura}
\authornote{Both authors contributed equally to this research.}
\email{nomura\_masahiro@cyberagent.co.jp}
\affiliation{%
  \institution{CyberAgent, Inc.}
  \state{Tokyo}
  \country{Japan}
}

\author{Yuta Saito}
\authornotemark[1]
\email{saito@hanjuku-kaso.com}
\affiliation{%
  \institution{Hanjuku-kaso, Co., Ltd.}
  \state{Tokyo}
  \country{Japan}
}

\renewcommand{\shortauthors}{Nomura and Saito}

\begin{abstract}
A typical assumption in supervised machine learning is that the train (source) and test (target) datasets follow completely the same distribution.
This assumption is, however, often violated in uncertain real-world applications, which motivates the study of learning under \textit{covariate shift}. 
In this setting, the naive use of adaptive hyperparameter optimization methods such as Bayesian optimization does not work as desired since it does not address the distributional shift among different datasets. 
In this work, we consider a novel hyperparameter optimization problem under the \textit{multi-source} covariate shift whose goal is to find the optimal hyperparameters for a target task of interest using only unlabeled data in a target task and labeled data in \textit{multiple} source tasks. 
To conduct efficient hyperparameter optimization for the target task, it is essential to estimate the target objective using only the available information. 
To this end, we construct the \textbfit{variance reduced estimator} that unbiasedly approximates the target objective with a desirable variance property.
Building on the proposed estimator, we provide a general and tractable hyperparameter optimization procedure, which works preferably in our setting with a no-regret guarantee. 
The experiments demonstrate that the proposed framework broadens the applications of automated hyperparameter optimization.
\end{abstract}



% \keywords{hyperparameter optimization, covariate shift, model selection, automated machine learning}

\maketitle

\section{Introduction}
% HPO
\textit{Hyperparameter optimization} (HPO) has been a pivotal part of machine learning (ML) and contributed to achieving a good performance in a wide range of tasks~\citep{feurer2019hyperparameter}.
It is widely acknowledged that the performance of deep neural networks depends greatly on the configuration of the hyperparameters~\citep{dacrema2019we,henderson2018deep,lucic2018gans}.
HPO is formulated as a special case of a black-box function optimization problem, where the input is a set of hyperparameters, and the output is a validation score. 
Among the black-box optimization methods, adaptive algorithms, such as \textit{Bayesian optimization} (BO)~\citep{brochu2010tutorial,shahriari2015taking,frazier2018tutorial} have shown superior empirical performance compared with traditional algorithms, such as grid search or random search~\citep{bergstra2012random,turner2021bayesian}.

One critical, but often overlooked assumption in HPO is the \textbfit{availability of an accurate validation score}.
However, in reality, there are many cases where we cannot access the ground-truth validation score of the target task of interest. 
For example, in display advertising, predicting the effectiveness of each advertisement, i.e., \textit{click-through rates} (CTR), is important for showing relevant advertisements (ads) to users. 
Therefore, it is necessary to conduct HPO before a new ad campaign starts. 
However, for new ads that have not yet been displayed to users, one cannot use labeled data to conduct HPO. 
Moreover, suppose you are a doctor who wants to optimize a medical plan for patients in a new, \textbfit{target} hospital (hospital A) using ML.
However, you have only historical medical outcome data from some \textbfit{source} hospitals (hospitals B and C) different from the target one.
In these situations, we have to tune hyperparameters that lead to a good ML model with respect to the target task without having its labeled data.
These settings are considered as a generalization of \textbfit{covariate shift}~\cite{shimodaira2000improving,sugiyama2007covariate}.
Covariate shift is a prevalent setting for supervised machine learning in the wild, where the feature distribution is different in the train (source) and test (target) tasks, but the conditional label distribution given the feature vector remains unchanged.
In this case, the standard HPO procedure is infeasible, as one cannot utilize the labeled target task data and the true validation score of the ML model under consideration.

In this work, we extend the use of adaptive HPO algorithms to \textbfit{multi-source covariate shift} (\textsl{MS-CS}).
In this setting, we do \textbfit{not} have labeled data for a target task (e.g., medical outcomes in a new hospital). 
However, we do have data for some source tasks following distributions that are different from the target task (e.g., historical outcomes in some other hospitals).
The difficulty to conduct HPO under MS-CS is that the ground-truth validation score is inaccessible, as we cannot utilize the label information of the target task. 
It is thus essential to find a good estimator of the target objective, a performance of an ML model with respect to the target task, calculable with only available data. 

A natural candidate approach is \textit{importance sampling}~\cite{elvira2015efficient,sugiyama2007covariate}, which leads to an unbiased estimation by correcting the distributional shifts among different tasks.
In our problem, however, a mere IS can lead to a sub-optimal conclusion, because the \textbfit{variance} in the objective estimation can blow up due to the effect of some source tasks greatly different from the target task.
To address this problematic variance issue, we first argue that the variance is the key to lead to an efficient HPO procedure in MS-CS.
We then propose the \textbfit{variance reduced (VR) estimator}, which achieves the optimal variance among a class of unbiased estimators by upweighting informative samples, i.e., samples of source tasks that are similar to the target task, based on a \textbfit{task divergence measure}. 
We also show that our proposed estimator leads to a \textbfit{no-regret} HPO even under MS-CS.
Finally, empirical studies on synthetic and real-world datasets demonstrate that the proposed framework works properly compared to other possible heuristics.

\paragraph{\textbf{Related Work}.}
A typical HPO aims to find a better set of hyperparameters implicitly assuming that samples from the target task are available. 
As faster convergence is an essential performance metric of the HPO methods, the research community is moving on to the \textit{multi-source} or \textit{transfer} settings for which there are some previously solved related source tasks. 
By combining the additional source task information and the labeled target task dataset, one can improve the hyperparameter search efficiency, and thus reach a better solution with fewer evaluations~\citep{bonilla2008multi,bardenet2013collaborative,swersky2013mtbo,yogatama2014efficient,ramachandran2018information,springenberg2016bayesian,poloczek2017multi,wistuba2018scalable,feurer2018scalable,perrone2018scalable,perrone2019learning,salinas2020quantile,nomura2021warm}.
A critical difference between multi-source HPOs and our MS-CS is the \textbfit{existence of labels for the target task}.
Previous studies usually assume that analysts can utilize the labeled target data. 
However, as discussed above, this is often unavailable, and thus, most of these methods are infeasible in practice.

% One possible solution to address the unavailability of labeled target data is to use warm starting methods~\citep{vanschoren2019meta}, which aims to find good initial hyperparameters for the target task.
% {\it Learning Initialization} (LI) finds promising hyperparameters by minimizing the sum of a loss function surrogated by a Gaussian process on each source task~\citep{wistuba2015learning}.
% While LI is effective when the source and target tasks are quite similar, it is hard to achieve a reasonable performance otherwise.
% In contrast, {\it DistBO} learns the similarity between the source and target tasks with a joint Gaussian process model on hyperparameters and data representations~~\citep{law2019hyperparameter}.
% However, many transfer methods including DistBO need abundant hyperparameter evaluations of the source tasks to surrogate the objective function of each task well, which limits their real-world applicability.
% From the theoretical point of view, LI and DistBO are not guaranteed to be no-regret, while our proposed method is under MS-CS.

Another related field is \textbfit{model evaluation under covariate shift}, whose objective is to evaluate the performance of ML models with respect to the target task using only a relevant \textbfit{single} source dataset~\citep{sugiyama2007covariate,you2019towards,zhong2010cross}. 
These studies build on the \textit{importance sampling} (IS) method~\citep{elvira2015efficient,sugiyama2007covariate} to obtain an unbiased estimate of ground-truth model performance. 
While our proposed methods are also based on IS, a major difference is that we assume that there are \textit{multiple} source datasets following different distributions. 
We demonstrate that, in the multi-source setting, the previous IS method can fail and propose an estimator satisfying the optimal variance property.
Moreover, as these methods are specific to \textit{model evaluation}, the connection between the IS-based estimation techniques and automated HPO methods has not yet been explored despite their possible broad applications. 
Consequently, we are the first to theoretically show the no-regret guarantee of adaptive HPO under MS-CS and empirically evaluate its possible combination with the IS-based unbiased estimation.

\paragraph{\textbf{Contributions}.} 
Our contributions are summarized as follows:
\begin{itemize}
    \item We formulate a novel HPO setting under MS-CS.
    \item We construct the VR estimator, which achieves the optimal variance among a reasonable class of unbiased estimators.
    \item We show that our proposed estimator leads to a \textbfit{no-regret} HPO even under MS-CS where labeled target task data are unavailable.
    \item We empirically demonstrate that the proposed procedure works favorably in MS-CS setting. 
\end{itemize}

\begin{figure}
\centering
\includegraphics[width=0.99\linewidth]{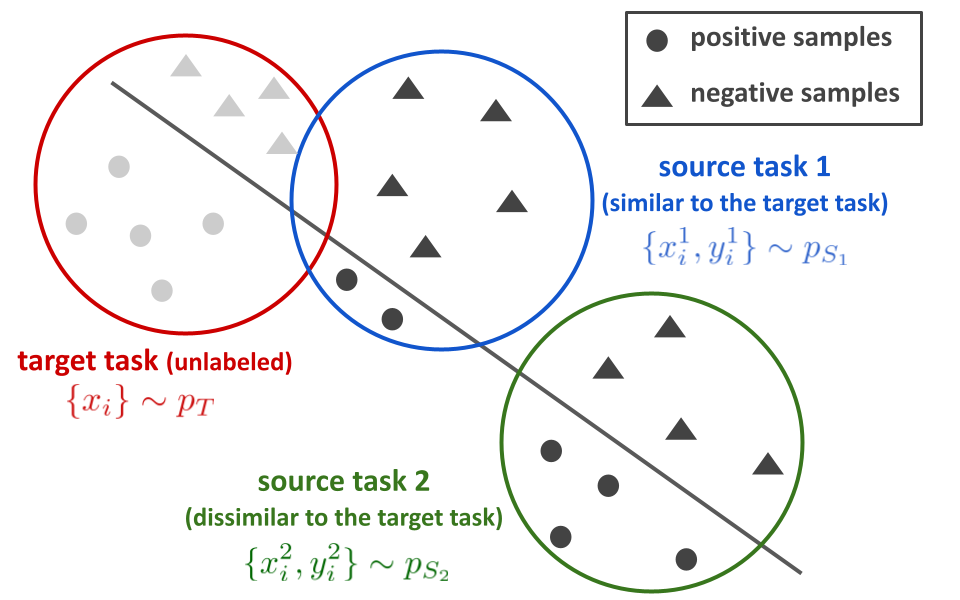}
\caption{\textit{Multi-Source} Covariate Shift (\textsl{MS-CS})}
\label{fig:gcs}
\vskip 0.05in
\begin{minipage}{\columnwidth} % choose width suitably
{\small {\it Note}: 
In MS-CS, we aim to find the best hyperparameters for the target task using unlabeled samples from the target task and labeled datasets from source tasks.
Some source tasks are similar to the target (e.g., source task 1), while others are dissimilar (e.g., source task 2). Note that the decision boundary is the same for all tasks (c.f., Assumption 2.).
\par}
\end{minipage}
\end{figure}

\section{Setup}
Here, we formulate HPO under MS-CS. 
Let $\mathcal{X} \subseteq \mathbb{R}^d$ be the $d$-dimensional input space and $\mathcal{Y} \subseteq \mathbb{R}$ be the real-valued output space. 
We use $p_T (x,y)$ to denote the joint probability density function of the input and output variables of the target task. 
The objective of this work is to find the best set of hyperparameters $\theta$ with respect to the target distribution:
\begin{align}
    \theta^{\ast} \in \argmin_{\theta \in \Theta} \tf, \label{eq:objective}
\end{align}
where $\Theta$ is a pre-defined hyperparameter search space and $\tf$ is the target task objective, which is defined as the generalization error over the target distribution:
\begin{align}
    \tf := \mE_{(X, Y) \sim p_T} \left[ L ( h_{\theta} (X), Y ) \right], \label{eq:target_task_objective}
\end{align}
where $L: \mathcal{Y} \times \mathcal{Y} \rightarrow \mathbb{R}_{\ge 0}$ is a bounded loss function such as the zero-one loss.
$h_{\theta}: \mathcal{X} \rightarrow \mathcal{Y}$ is an arbitrary machine learning model defined by a set of hyperparameters $\theta \in \Theta$, which predicts the output values using the input vectors.

In a standard hyperparameter optimization setting~\cite{bergstra2011algorithms,feurer2019hyperparameter,snoek2012practical}, labeled i.i.d. validation samples $ \calD_T^{labeled} := \{ x_i, y_i \}_{i=1}^{n_T} \sim p_T $ are available, and one can easily estimate the target objective in \Eqref{eq:target_task_objective} by the following empirical mean:
\begin{align}
    \hat{f}_T (\theta; \calD_T^{labeled}) := \frac{1}{n_T} \sum_{i=1}^{n_T} L ( h_{\theta} (x_i), y_i).  \label{eq:empirical_target_task_objective}
\end{align}
A hyperparameter optimization is conducted directly minimizing \Eqref{eq:empirical_target_task_objective} as a reasonable replacement for the ground-truth objective $\tf$ in \Eqref{eq:target_task_objective}.

In contrast, under MS-CS, labels of the target task are unobservable; we can use only \textbfit{unlabeled} target validation samples denoted as $\calD_T := \{ x_i \}_{i=1}^{n_T} $.
Instead, we can utilize \textit{multiple} source task datasets, i.e., $\{ \calD_{S_j} \}_{j=1}^{k}$ where $j$ is a source task index and $k$ denotes the number of source tasks. 
Each source task is defined as i.i.d. \textbfit{labeled} samples: $\calD_{S_j} := \{ \ijx, \ijy\}_{i=1}^{n_{S_j}} \sim p_{S_j} $ where $ p_{S_j} (x,y)$ is a joint probability density function that characterizes the source task $j$. 
Note here that the marginal input distributions of the target and source tasks are different, i.e., $p_T (x) \neq p_{S_j} (x), \; \forall j \in \{1, \ldots, k\}$.
Finally, let we use $n:=\sum_{j=1}^{k} n_{S_j}$ to denote the total number of samples in the source tasks.

Regarding the target and source distributions, we make the following assumptions.
\begin{assumption} Source tasks have support for the target task, i.e., $p_T (x) > 0 \Rightarrow p_{S_j} (x) > 0, \; \forall x \in \mathcal{X}, \, \forall j \in \{1, \ldots, k\} $.
\label{assumption1}
\end{assumption}

\begin{assumption} Conditional output distributions remain the same between the target and all the source tasks, i.e., $p_T (y | x) = p_{S_j} (y | x),$ $\; \forall j \in \{1, \ldots, k\} $.
\label{assumption2}
\end{assumption}

The above assumptions are common in the \textit{covariate shift} literature~\cite{shimodaira2000improving} and suggest that the input-output relation is the same, but the input distributions are different for the target and source task distributions.\footnote{These assumptions seem to be strict, but in fact, they are relatively reasonable given that the general HPO literature implicitly assumes that the train-test distributions are completely the same, i.e., $p_T (x,y) = p_{S_j} (x,y),\; \forall j \in \{1, \ldots, k\}$.}

One critical difficulty of MS-CS is that the simple approximation by the empirical mean in Eq.~(\ref{eq:empirical_target_task_objective}) is infeasible, as the labeled target dataset is unavailable.
It is thus essential to estimate the target task objective $f_T$ accurately using only an unlabeled target dataset and labeled multiple source datasets.

\section{Method} \label{sec:method}
In this section, we first define a class of unbiased estimators for the target objective by applying \textit{importance sampling} to our setting.
Any estimator within this class provides an unbiased target objective. However, optimizing some of them can be sub-optimal, as they can have a large variance when the target and source distributions differ greatly. To overcome the variance issue, we propose the \textbfit{variance reduced estimator}, which achieves the minimum achievable variance among the class of unbiased estimators. 

\subsection{A Class of Unbiased Estimators} \label{sec:unbiased_estimators}
A natural method to approximate the target task objective under distributional shifts is to use \textit{importance sampling}~\cite{shimodaira2000improving}.
It addresses the distributional shifts among different tasks using the following density ratio.
\begin{definition} (Density Ratio) For any $(x,y) \in \mathcal{X} \times \mathcal{Y}$ with a positive source density $p_{S_j} (x,y) > 0$, the density ratio between the target and the source task distributions is
\begin{align*}
    w_{S_j} (x, y) := \frac{p_T (x, y)}{p_{S_j} (x,y)} = \frac{p_T (x)}{p_{S_j} (x)} = w_{S_j}(x)
\end{align*}
where $w_{S_j}(x) \in [0,C] $ for a positive constant $C$. The equalities are derived from \Asmref{assumption2}.
\end{definition}

Using the density ratio, we define a class of unbiased estimators for the target task objective.
\begin{definition} ($\boldsymbol{\lambda}$-unbiased Estimator) 
For a given set of hyperparameters $\theta \in \Theta$, the class of $\boldsymbol{\lambda}$-unbiased estimators for the target task objective is defined as
\begin{align}
  \lamf := \sum_{j=1}^{k}  \lambda_j \sum_{i=1}^{n_{S_j}} w_{S_j} (\ijx) \cdot L ( h_{\theta} (\ijx), \ijy ),  \label{eq:lam_ub}
\end{align}
where $ \boldsymbol{\lambda} = \{ \lambda_1, \ldots \lambda_{k}\} $ is a set of weights for source tasks that satisfies $ \lambda_j \ge 0 $ and $\sum_{j=1}^{k} \lambda_j n_{S_j}  = 1  $ for all $j \in \{1, \ldots k\}$.
\end{definition}

This estimator is a generalization of the \textit{importance weighted cross-validation}~\cite{sugiyama2007covariate} to the multi-source setting and is statistically unbiased for the ground-truth target task objective, i.e., for any given $\theta$ and $\boldsymbol{\lambda}$, we have $ \mE [ \lamf ] = \tf$. This is derived as follows.
\begin{align*}
    \mE & \left[ \hat{f}_{\boldsymbol{\lambda}} \left(\theta; \{ \D_{S^j} \}_{j=1}^{k}\right) \right] \\
    & = \sum_{j=1}^{k}  \lambda_j \sum_{i=1}^{n_{S^j}} \mE_{(X, Y) \sim p_{S^j}} \left[ w_{S^j} (X) \cdot L ( h_{\theta} (X), Y )  \right] \\
    & = \sum_{j=1}^{k}  \lambda_j \sum_{i=1}^{n_{S^j}} \mE_{(X, Y) \sim p_{S^j}} \left[ \frac{p_T (X, Y)}{p_{S^j} (X,Y)} \cdot L ( h_{\theta} (X), Y )  \right] \\
    & = \sum_{j=1}^{k}  \lambda_j \sum_{i=1}^{n_{S^j}} \mE_{(X, Y) \sim p_T} \left[ L ( h_{\theta} (X), Y )  \right] \\
    & = \sum_{j=1}^{k}  \lambda_j \sum_{i=1}^{n_{S^j}} \tf \\
    & = \tf.
\end{align*}
% \end{proof}

Using an instance of the class of $\boldsymbol{\lambda}$-unbiased estimators is a valid approach for approximating the target task objective because of its unbiasedness.
However, a question that comes up to our minds here is: 

\begin{quote}
    \textbfit{What specific estimator should we use among the class of $\boldsymbol{\lambda}$-unbiased estimators?}
\end{quote}

A key property to answer this question is the \textbfit{variance} of the estimators.
For instance, let us define an immediate \textit{unbiased estimator} $\ubf := n^{-1} \sum_{j=1}^{k}  \sum_{i=1}^{n_{S_j}} w_{S_j} (\ijx) L ( h_{\theta} (\ijx), \ijy )$. 
This is a version of the $\boldsymbol{\lambda}$-unbiased estimators with $\lambda_j = n_{S_j}/n$ for every $j$.
We investigate the variance of this instance of the $\boldsymbol{\lambda}$-unbiased estimators below.
\begin{proposition} (Variance of $\hat{f}_{UB}$) For a given set of hyperparameter $\theta \in \Theta$, the variance of $\hat{f}_{UB}$ is
\begin{align}
    &\mV \left( \ubf \right) \nonumber \\
    &= \frac{1}{n^2} \sum_{j=1}^{k}  n_{S^j} \left( \mE_{(X, Y) \sim p_{S^j} } \left[  w_{S^j}^2(X) \cdot L^2 ( h_{\theta} (X), Y )  \right] - ( \tf )^2   \right) \label{var_ub}
\end{align}

\begin{proof}
Because the samples are independent, the variance can be represented as follows.
\begin{align*}
   \mV \left( \ubf \right)  
   & =  \frac{1}{n^2} \sum_{j=1}^{k}  \sum_{i=1}^{n_{S^j}} \mV \left( w_{S^j} (X) \cdot L ( h_{\theta} (X), Y ) \right) \\
   & =  \frac{1}{n^2} \sum_{j=1}^{k}  n_{S^j} \cdot \mV \left( w_{S^j} (X) \cdot L ( h_{\theta} (X), Y ) \right)
\end{align*}
where $\mV \left( w_{S^j} (X) \cdot L ( h_{\theta} (X), Y ) \right)$ can be decomposed as 
\begin{align*}
    \mV & \left( w_{S^j} (X) \cdot L ( h_{\theta} (X), Y ) \right) = \mE_{(X, Y) \sim p_{S^j}} \left[ w^2_{S^j} (X) \cdot L^2 ( h_{\theta} (X), Y ) \right] \nonumber \\
        & -  \left( \mE_{(X, Y) \sim p_{S^j}} \left[  w_{S^j} (X) \cdot L ( h_{\theta} (X), Y )  \right] \right)^2,
\end{align*}
From unbiasedness, $\mE_{(X, Y) \sim p_{S^j}} \left[  w_{S^j} (X) \cdot L ( h_{\theta} (X), Y )  \right]  = \tf$.
Thus, we obtain \Eqref{var_ub}.
\end{proof}
\label{prop_var_of_ub}
\end{proposition}

The problem here is that, according to Proposition~\ref{prop_var_of_ub}, its variance depends on the square value of the density ratio function, which can be large when there is a source task having a distribution that is dissimilar to that of the target task. 

To illustrate this variance problem, we use a toy example given by $\{ x_1, x_2 \} \subseteq \mathcal{X}, \,  \{y_1, y_2\} \subseteq \mathcal{Y}, \, p(y_1 | x_1 ) = p(y_2 | x_2 ) = 1, \,  p(y_2 | x_1 ) = p(y_1 | x_2 ) = 0$. 
The loss values for possible tuples and the probability densities of the target and two source tasks are presented in \Tabref{tab1}. 
It shows that the target task $T$ is similar to the source task $S^2$, but its distribution is significantly different from that of $S^1$. For simplicity and without loss of generality, suppose there are two source task datasets given by $ \D^1 = \{(x_1^1, y_1^1)\} $ and $ \D^2 = \{(x_1^2, y_1^2)\} $. 
Then from \Eqref{var_ub}, the variance of the unbiased estimator is about $64.27$. 
Intuitively, this large variance is a result of the large variance samples from $ S^1 $. 
In fact, by dropping the samples of $S^1$ reduces the variance to $4.27$. 
From this example, we know that the unbiased estimator fails to make the most of the source tasks, and there is room to improve its variance by down-weighting the source tasks that are dissimilar to the target task.

\begin{table}[ht]
\centering
\caption{Dropping data samples from $S^1$ significantly lowers the variance of the unbiased estimator}
% \scalebox{0.8}{
\def\arraystretch{1.1}
\begin{tabular}{lcc}
\toprule
 & $ (x_1, y_1) $ & $ (x_2, y_2) $    \\ \midrule
loss function: $ L (h_{\theta} (x), y) $ & 10 & 1     \\ \midrule
target task ($T$) distribution: $ p_T (x,y) $ & 0.8 & 0.2     \\
source task ($S^1$) distribution: $ p_{S^1} (x,y) $ & 0.2 & 0.8  \\
source task ($S^2$) distribution: $ p_{S^2} (x,y) $ & 0.9 & 0.1   \\ \bottomrule
\end{tabular}
% }
\label{tab1}
\end{table}

\subsection{Variance Reduced Objective Estimator}
To deal with the potential large variance of the $\boldsymbol{\lambda}$-unbiased estimators, we consider re-weighting samples from various source tasks based on their similarity to the target task. 
This builds on the idea that samples from source tasks close to the target task are much more informative in HPO than samples from dissimilar source tasks.
To lead to a more reasonable estimator, we first define a \textbfit{task divergence measure}, which quantifies the similarity between the two tasks below. 
\begin{definition}(Task Divergence Measure)
The divergence between a source task distribution $p_{S_j}$ and the target task distribution $p_T$ is defined as
\begin{align}
    \jsim := \mE_{p_{S_j}} \left[w_{S_j}^2(X) \cdot L^2 ( h_{\theta} (X), Y )  \right] - ( \tf )^2 \label{task_divergence}
\end{align}
\end{definition}

This task divergence measure is large when the corresponding source distribution deviates significantly from the target task distribution. 
Building on this measure, we define the following estimator of the target task objective.
\begin{definition} (Variance Reduced Estimator) 
For a given set of hyperparameters $\theta \in \Theta$, the variance reduced estimator of the target task objective function is defined as
\begin{align}
   \vrf = \sum_{j=1}^{k}  \lambda^{\star}_j \sum_{i=1}^{n_{S^j}} w (\ijx) \cdot L ( h_{\theta} (\ijx), \ijy ) \label{vr}
\end{align}
where VR stands for variance reduced, $\D_{S^j}$ is any sample size $n_{S^j}$ of the i.i.d. samples from source task $j$, and $w_{S^j} (\cdot)$ is the true density ratio.
$\lambda^{\star}_j$ is a weight for source task $j$, which is defined as $$\lambda^{\star}_j =  \left( \jsim \sum_{j=1}^{k} \frac{n_{S^j}}{\jsim} \right)^{-1}$$
Note that, for all $j \in \{1, \ldots k\}$, $ \lambda^{\star}_j \ge 0 $ and $\sum_{j=1}^{k} \lambda^{\star}_j n_{S^j}  = 1 $.
\end{definition}
The variance reduced estimator in \Eqref{vr} is statistically unbiased for the ground-truth target task objective in \Eqref{eq:target_task_objective}, i.e., for any given $\theta$, $ \mE [ \vrf ] = \tf$, as it is an instance of the $\boldsymbol{\lambda}$-unbiased estimators.

Then, we demonstrate that the variance reduced estimator in \Eqref{vr} is \textbfit{optimal} in the sense that any other convex combination of a set of weights $\boldsymbol{\lambda} = \{ \lambda_1, \ldots \lambda_{k}\}$ that satisfies unbiasedness cannot provide a smaller variance.

\begin{theorem} (Variance Optimality; Extension of Theorem 6.4 of~\citep{agarwal2017effective}) For any given set of weights $ \boldsymbol{\lambda} = \{ \lambda_1, \ldots \lambda_{k}\} $ that satisfies $ \lambda_j \ge 0 $ and $\sum_{j=1}^{k} \lambda_j n_{S^j}  = 1  $ for all $j \in \{1, \ldots k\}$, the following inequality holds
\begin{align*}
    \mV \left( \vrf \right) \le  \mV \left( \hat{f}_{\boldsymbol{\lambda}} \left(\theta; \{\D_{S^j} \}_{j=1}^{k} \right) \right)
\end{align*}
where $ \hat{f}_{\boldsymbol{\lambda}} (\theta; \{\D_{S^j} \}_{j=1}^{k}) = \sum_{j=1}^{k}  \lambda_j \sum_{i=1}^{n_{S^j}} w (\ijx) \cdot L ( h_{\theta} (\ijx), \ijy )$.
\label{theorem1}

\begin{proof}
By following the same logic flow as in Proposition~\ref{prop_var_of_ub}, the variance of the $\boldsymbol{\lambda}$-unbiased estimator in \Eqref{eq:lam_ub} is
\begin{align}
   &\mV \left( \hat{f}_{\boldsymbol{\lambda}} \left(\theta; \{ \D_{S^j} \}_{j=1}^{k}\right) \right) \nonumber \\
   & =  \sum_{j=1}^{k} \lambda_j^2 n_{S^j}  \left( \mE_{(X, Y) \sim p_{S^j}} \left[ w^2_{S^j} (X) \cdot L^2 ( h_{\theta} (X), Y ) \right] - \left( \tf \right)^2 \right) \notag \\
   & = \sum_{j=1}^{k} \lambda_j^2  n_{S^j} \cdot \jsim \notag % \label{var_lam_ub}
\end{align}
Thus, by replacing $\lambda_j$ for $ \left( \jsim \sum_{j=1}^{k} \frac{n_{S^j}}{\jsim} \right)^{-1} $, we have
\begin{align*}
    \mV \left( \hat{f}_{\boldsymbol{\lambda}} \left(\theta; \{ \D_{S^j} \}_{j=1}^{k}\right) \right) 
    & = \sum_{j=1}^{k} \left(\sum_{j=1}^{k} \frac{n_{S^j}}{\jsim} \right)^{-2} n_{S^j}  \cdot \jsim \\
    & = \sum_{j=1}^{k} \frac{n_{S^j} \jsim}{(\jsim)^2 (\sum_{j=1}^{k} \frac{n_{S^j}}{\jsim})^2 } \\
   & = \left(\sum_{j=1}^{k} \frac{n_{S^j}}{\jsim} \right) \left(\sum_{j=1}^{k} \frac{n_{S^j}}{\jsim} \right)^{-2} \\
   & = \left(\sum_{j=1}^{k} \frac{n_{S^j}}{\jsim} \right)^{-1}
\end{align*}
Then, for any set of weights $ \boldsymbol{\lambda} = \{ \lambda_1, \ldots \lambda_{k}\} $, we obtain the following variance optimality using the Cauchy-Schwarz inequality.
\begin{align*}
    & \left( \sum_{j=1}^{k} \lambda_j^2  n_{S^j} \cdot \jsim \right) \left( \sum_{j=1}^{k} \frac{n_{S^j}}{\jsim} \right)
    \ge \left( \sum_{j=1}^{k} \lambda_j  n_{S^j} \right)^2 = 1 \\
    & \Longrightarrow \left( \sum_{j=1}^{k} \lambda_j^2  n_{S^j} \cdot \jsim \right) \ge \left(\sum_{j=1}^{k} \frac{n_{S^j}}{\jsim} \right)^{-1} \\
    & \Longrightarrow  \mV \left( \hat{f}_{\boldsymbol{\lambda}} \left(\theta; \{ \D_{S^j} \}_{j=1}^{k}\right) \right) \ge \mV \left( \vrf \right) 
\end{align*}
\end{proof}
\end{theorem}

\Thmref{theorem1} suggests that the variance reduced estimator achieves a desirable variance property by weighting each source task based on its divergence from the target task.

Let us now return to the toy example in \Tabref{tab1}.\footnote{We refer to~\citep{agarwal2017effective} to create the toy example.}
The values of the divergence measure for $S^1$ and $S^2$ are $Div(T || S^1)=252.81$ and $Div(T || S^2)=4.27$, respectively. This leads to the weights of $ \lambda_1^{\star} \approx  0.017$ and $ \lambda_2^{\star} \approx 0.983$. 
Then, we have $4.21 \; (\text{variance of } \hat{f}_{VR}) > 4.27 \; (\text{variance when $S^1$ is dropped}) >64.27 \; (\text{variance of } \hat{f}_{UB}) $.
It is now obvious that the variance reduced estimator $\hat{f}_{VR}$ performs much better than the unbiased estimator by optimally weighting all available source tasks.

\begin{algorithm*}[t]
\caption{Hyperparameter optimization procedure for the MS-CS setting}
\begin{algorithmic}[1]
\Require unlabeled target task dataset $\mathcal{D}_T = \{x_i\}_{i=1}^{n_T}$; labeled source task datasets  $ \{ \mathcal{D}_{S^j} = \{ x_i^j, y_i^j \}_{i=1}^{n_{S^j}} \}_{j=1}^{k}$; hyperparameter search space $\Theta$; a machine learning model $h_{\theta}$; a target task objective estimator $\hat{f}$, a hyperparameter optimization algorithm $\textbf{OPT}$
\For{ $j \in \{1, \ldots, k\}$ }
    \State Split $\mathcal{D}_{S^j}$ into three folds $\mathcal{D}_{S^j}^{density} $, $\mathcal{D}_{S^j}^{train}$, and $\mathcal{D}_{S^j}^{val}$
    \State Estimate density ratio $ w_{S^j} (\cdot) $ by uLSIF with $\mathcal{D}_T$ and $\mathcal{D}_{S^j}^{density}$
\EndFor
\State Optimize the hyperparameter $\theta \in \Theta$ of $h_{\theta}$ with $\textbf{OPT}$ by setting $ \hat{f} ( \theta;  \{ \mathcal{D}_{S^j}^{val} \}_{j=1}^{k}) $ as its objective 
\State \quad (the model parameter of $h_{\theta}$ is obtained by optimizing $ \hat{f} ( \theta;  \{ \mathcal{D}_{S^j}^{train} \}_{j=1}^{k}) $)\\
\Return $h_{\theta^{\star}}$ (where $\theta^{\star}$ is the output of $\textbf{OPT}$)
\end{algorithmic}
\end{algorithm*}

\subsection{Hyperparameter Optimization Procedure}
We describe several detailed components of the proposed HPO procedure for the MS-CS setting.

\paragraph{\textbf{Density Ratio Estimation:} }
In general, the density ratio functions between the target and source tasks are unavailable and thus should be estimated beforehand.
To estimate the density ratio, we employ \textit{unconstrained Least-Squares Importance Fitting (uLSIF)}~\cite{kanamori2009least,yamada2011relative}, which minimizes the following squared error for the true density ratio function:
\begin{align}
    \hat{s} 
    & =  \argmin_{s \in \mathcal{S}} \, \mE_{p_{S^j}} \left[ \left( w(X) - s (X) \right)^2 \right] \nonumber \\
    & = \argmin_{s \in \mathcal{S}} \, \left[ \frac{1}{2} \mE_{p_{S^j}} \left[ s^2 (X) \right] - \mE_{p_T} [ s(X) ] \right]
    \label{lsif}
\end{align}
where $\mathcal{S}$ is a class of measurable functions. 
Note that the empirical version of \Eqref{lsif} is calculable with only unlabeled target and source task datasets.

\paragraph{\textbf{Task Divergence Estimation:} }
To utilize the variance reduced estimator, the task divergence measure $\jsim$ in \Eqref{task_divergence} needs to be estimated from the available data. This can be done using the following empirical mean.
\begin{align}
  \widehat{Div} \left( T \, || \, S^j \right)
  = &\frac{1}{n_{S^j}} \sum_{i=1}^{n_{S^j}} \left( w (\ijx) \cdot L ( h_{\theta} (\ijx), \ijy ) \right)^2 \nonumber \\
  &- \left( \frac{1}{n_{S^j}} \sum_{i=1}^{n_{S^j}} w (\ijx) \cdot L ( h_{\theta} (\ijx), \ijy )  \right)^2
\end{align}

\paragraph{\textbf{Data Splitting: }}
To conduct HPO in our setting, we split the dataset of each source task $\mathcal{D}_{S^j}$ into three folds $\mathcal{D}_{S^j}^{density} $, $\mathcal{D}_{S^j}^{train}$, and $\mathcal{D}_{S^j}^{val}$ where each corresponds to the data for the density ratio estimation, training $h_{\theta}$, and evaluating $h_{\theta}$, respectively. First, we use $\mathcal{D}_{S^j}^{density}$ and $\mathcal{D}_T$ to estimate the density ratio function between the source task $j$ and the target task. Second, we use $\mathcal{D}_{S^j}^{train}$ to train $h_{\theta}$ for a given hyperparameter $\theta$ in an HPO procedure. Finally, we use $\mathcal{D}_{S^j}^{val}$ and $\mathcal{D}_T$ to estimate the target task objective and evaluate the performance of $\theta$. The data splitting is needed to ensure the variance optimality of our estimator.

\paragraph{\textbf{Overall Procedure:} } 
Building on the above details, Algorithm 1 summarizes the high-level hyperparameter optimization procedure, which we develop for the MS-CS setting.
In addition, we describe the specific hyperparameter optimization procedure when BO is used as \textbf{OPT} in Algorithm 2.

\begin{algorithm*}[t]
\caption{Bayesian Optimization for the MS-CS setting}
\begin{algorithmic}[1]
\Require unlabeled target task dataset $\mathcal{D}_T = \{x_i\}_{i=1}^{n_T}$; labeled source task datasets  $ \{ \mathcal{D}_{S^j} = \{ x_i^j, y_i^j \}_{i=1}^{n_{S^j}} \}_{j=1}^{k}$; hyperparameter search space $\Theta$; a machine learning model $h_{\theta}$; a target task objective estiamtor $\hat{f}$, number of evaluations $B$, acquisition function $\alpha (\cdot)$
\Ensure the optimized set of hyperparameters $\theta^{\star} \in \Theta$
\State Set $ \mathcal{A}_0 \leftarrow \emptyset $
\For{ $j \in \{1, \ldots, k\}$ }
    \State Split $\mathcal{D}_{S^j}$ into three folds $\mathcal{D}_{S^j}^{density} $, $\mathcal{D}_{S^j}^{train}$, and $\mathcal{D}_{S^j}^{val}$
    \State Estimate density ratio $ w_{S^j} (\cdot) $ by uLSIF with $\mathcal{D}_T$ and $\mathcal{D}_{S^j}^{density}$
\EndFor
\For{$t=1,2, \ldots, B$}
    \State Select $\theta_t$ by optimizing $\alpha (\theta \, | \, \mathcal{A}_{t-1}) $
    \State Train $h_{\theta_t}$ by optimizing $ \hat{f} ( \theta;  \{ \mathcal{D}_{S^j}^{train} \}_{j=1}^{k}) $ and obtain a trained model $h_{\theta}^{*}$
    \State Evaluate $ h_{\theta}^{*} $ and obtain a validation score $z_t = \hat{f} (\theta;  \{ \mathcal{D}_{S^j}^{val} \}_{j=1}^{k}) $
    \State $ \mathcal{A}_t \leftarrow \mathcal{A}_{t-1} \cup \{ (\theta_t, z_t) \} $
\EndFor
\State$t^{\star} = \argmin_t \{z_1, \ldots z_B \}$\\
\Return $h_{\theta^{\star}}$ (where $\theta^{\star} = \theta_{t^{\star}}$)
\end{algorithmic}
\end{algorithm*}

\section{Regret Analysis} \label{sec:regret}

In this section, we analyze the regret bound under the MS-CS setting and prove that our proposed estimator achieves \textbfit{no-regret} HPO.
Here, we define a \textit{regret} as
\begin{align*}
    r_{B}^{n} = f(\hat{\theta}_{B}^{\ast}) - f(\theta^{\ast}),
\end{align*}
where $f: \Theta \rightarrow \mathbb{R} $ is the ground-truth target task objective, $n = \sum_{j=1}^{k} n_{S^j}$ is the total sample size among source tasks, and $B$ is the total number of evaluations. We also let $\theta^{\ast} \in \argmin_{\theta \in \Theta} f(\theta)$, and $\hat{\theta}_{B}^{\ast} \in \argmin_{\theta \in \{ \theta_1, \cdots, \theta_B \}} \hat{f}_{n}(\theta)$ where $\hat{f}_{n} : \Theta \rightarrow \mathbb{R}$ is a target task objective estimated by the $\boldsymbol{\lambda}$-unbiased estimator.
Note that each of $\{ \theta_1, \cdots, \theta_B \}$ is the hyperparameter selected among $B$ evaluations in the optimization.

To derive a regret bound, we first decompose the regret into the following three terms:
\begin{align}
  r_{B}^{n} &= f(\hat{\theta}_{B}^{\ast}) - f(\theta^{\ast}) \nonumber \\
    &= (f(\hat{\theta}^{\ast}_{B}) - \hat{f}_{n}(\hat{\theta}^{\ast}_{B})) + \hat{f}_{n}(\hat{\theta}^{\ast}_{B}) + (\hat{f}_{n}(\hat{\theta}^{\ast}) - f(\theta^{\ast})) - \hat{f}_{n}(\hat{\theta}^{\ast}) \nonumber \\
    &= \underbrace{(\hat{f}_{n}(\hat{\theta}^{\ast}_{B}) - \hat{f}_{n}(\hat{\theta}^{\ast}))}_{(A)} + \underbrace{(f(\hat{\theta}^{\ast}_{B}) - \hat{f}_{n}(\hat{\theta}^{\ast}_{B}))}_{(B)} + \underbrace{(\hat{f}_{n}(\hat{\theta}^{\ast}) - f(\theta^{\ast}))}_{(C)}, \label{regret_bound_ABC}
\end{align}
where $\hat{\theta}^{\ast} \in \argmin_{\theta \in \Theta} \hat{f}_{n}(\theta)$.

The term (A) is the \emph{simple} regret obtained by optimizing the estimated target task objective $\hat{f}_{n}$.
The term (B) is the difference between the true objective $f$ and the estimated objective $\hat{f}_{n}$ at $\hat{\theta}^{\ast}_{B}$, which is the solution of $\hat{f}_n (\theta)$.
The term (C) is the difference between the minimum value of the estimated objective $\hat{f}_{n}$ and that of the true objective $f$.
This decomposition is essential because we can bound the overall regret by finding a bound for each term.

We first state two lemmas, which will be used to bound the regret.

\begin{lemma}
\label{lemma:B}
The following inequality holds with a probability of at least $ 1 - \delta \, (\delta \in (0, 1)),$
\begin{align*}
    (f(\hat{\theta}^{\ast}_{B}) - \hat{f}_{n}(\hat{\theta}^{\ast}_{B})) \leq \sqrt{ \mathbb{V}(\hat{f}_{n}(\hat{\theta}_B^{\ast})) / \delta}.
\end{align*}
\begin{proof}
By Chebyshev's inequality, we have
\begin{align*}
    \mP \{ f(\hat{\theta}_B^{\ast}) - \hat{f}_{n}(\hat{\theta}_B^{\ast}) \geq c \}
    &\leq \mP \{ | f(\hat{\theta}_B^{\ast}) - \hat{f}_{n}(\hat{\theta}_B^{\ast}) | \geq c \} \\
    &\leq \mathbb{V}(\hat{f}_{n}(\hat{\theta}_B^{\ast})) / c^2.\\
\end{align*}
Putting the RHS as $\delta$ and solving it for $c$ completes the proof.
\end{proof}
\label{boundB}
\end{lemma}

\begin{lemma} 
\label{lemma:C}
The following inequality holds with a probability of at least $ 1 - \delta \, (\delta \in (0, 1)),$
\begin{align*}
\hat{f}_{n}(\hat{\theta}^{\ast}) - f(\theta^{\ast}) \leq \sqrt{(\mathbb{V}(\hat{f}_{n}(\theta^{\ast})) + \mathbb{V}(\hat{f}_{n}(\hat{\theta}^{\ast}))) / \delta}.
\end{align*}
\begin{proof}
By Chebyshev's inequality, we have
\begin{align*}
  &\mP \{ \hat{f}_{n}(\hat{\theta}^{\ast}) - f(\theta^{\ast}) \geq c \} \nonumber \\
  &\leq \mP \{ | \hat{f}_{n}(\hat{\theta}^{\ast}) - f(\theta^{\ast}) | \geq c \} \nonumber \\
  &\leq \mP \{ | \hat{f}_{n}(\theta^{\ast}) - f(\theta^{\ast}) | \geq c \cup | \hat{f}_{n}(\hat{\theta}^{\ast}) - f(\hat{\theta}^{\ast})| \geq c \}  \\
  &\leq \mP \{ | \hat{f}_{n}(\theta^{\ast}) - f(\theta^{\ast}) | \geq c \} + \mP \{ | \hat{f}_{n}(\hat{\theta}^{\ast}) - f(\hat{\theta}^{\ast})| \geq c \}  \\
  &\leq \frac{1}{c^2} (\mathbb{V}(\hat{f}_{n}(\theta^{\ast})) + \mathbb{V}(\hat{f}_{n}(\hat{\theta}^{\ast}))).
\end{align*}
Putting the RHS as $\delta$ and solving it for $c$ completes the proof.
\end{proof}
\label{boundC}
\end{lemma}

Then, we state the main theorem.

\begin{theorem}(Regret Bound for MS-CS)  When the $\boldsymbol{\lambda}$-unbiased estimator with an arbitrary set of weights $\boldsymbol{\lambda}$ is used as $\hat{f} (\theta;  \{ \mathcal{D}_{S^j}^{val} \}_{j=1}^{k})$, the following regret bound holds with a probability of at least $ 1 - \delta \, (\delta \in (0, 1)),$
\begin{align}
  \label{eq:regret_theorem}
  r_{B}^{n} \leq R_n + \sqrt{2 \mathbb{V}(\hat{f}_{n}(\hat{\theta}_B^{\ast})) / \delta} + \sqrt{2 (\mathbb{V}(\hat{f}_{n}(\theta^{\ast})) + \mathbb{V}(\hat{f}_{n}(\hat{\theta}^{\ast}))) / \delta},
\end{align}
where $R_n = \hat{f}_{n}(\hat{\theta}^{\ast}_{B}) - \hat{f}_{n}(\hat{\theta}^{\ast})$.
\end{theorem}

\begin{proof}
Putting Lemma~\ref{lemma:B} to the term (B) in Eq.~(\ref{regret_bound_ABC}) and Lemma~\ref{lemma:C} to the term (C) in Eq.~(\ref{regret_bound_ABC}) complete the proof.
\end{proof}

When $\hat{f}_n$ is the proposed variance reduced estimator, $\mathbb{V}(\hat{f}_{n}(\cdot)) = o(1)$.
This means that the second and third terms in Eq.~(\ref{eq:regret_theorem}) are {\emph no-regret}~\citep{srinivas2010gaussian} with respect to $n$, i.e.,  $$\lim_{n \to \infty} \sqrt{2 \mathbb{V}(\hat{f}_{n}(\hat{\theta}_B^{\ast})) / \delta} + \sqrt{2 (\mathbb{V}(\hat{f}_{n}(\theta^{\ast}) + \mathbb{V}(\hat{f}_{n}(\hat{\theta}^{\ast}))) / \delta}) = 0.$$
Therefore, if the HPO optimization method is no-regret with respect to the number of evaluations $B$, i.e., $R_n = o(1)$, the overall regret $r_{B}^{n}$ diminishes as $n$ and $B$ increase.
Indeed, under some mild conditions, we can achieve $R_n = o(1)$ using no-regret optimization methods such as GP-UCB~\citep{srinivas2010gaussian} and GP-MES~\citep{wang2017max}.
Srinivas et al.~\citep{srinivas2010gaussian} indicate that GP-UCB is no-regret in the sense of \emph{cumulative} regret, and it is straightforward to apply the theoretical result to our notion of regret as done in \cite{kandasamy2018parallelised,kandasamy2019multi}.

\section{Experiments}
We investigate the behavior of our proposed HPO procedure in MS-CS using a toy problem in Section~\ref{sec:toy} and real-world datasets in Section~\ref{sec:hpo_real}.
The code for reproducing the results of the experiments is available at \href{https://github.com/nmasahiro/MS-CS}{\textbf{https://github.com/nmasahiro/MS-CS}}.

\subsection{Baseline Methods}
We compare the following four methods as possible baselines for MS-CS:
\textbf{(i)} {\it Learning Initialization (LI)}~\citep{wistuba2015learning},
\textbf{(ii)} {\it DistBO}~\citep{law2019hyperparameter},
\textbf{(iii):} {\it Naive} method, which uses the performance on the concatenation of source tasks as a validation score, 
\textbf{(iv)} {\it Oracle} method, which uses the labeled target task for HPO. Note that the oracle method is infeasible in MS-CS, and we regard the performance of the oracle method as an upper bound to which other methods can reach.

% LI
Learning Initialization (LI)~\cite{wistuba2015learning} aims to find promising hyperparameters by minimizing a meta-loss function.
The meta-loss function is defined as the sum of a surrogated loss function on each source task.
Intuitively, by minimizing the meta-loss function, LI can obtain hyperparameters that can lead to a good performance on average for source tasks.

% DistBO
DistBO transfers knowledge across tasks using learnt representations of training datasets~\citep{law2019hyperparameter}.
To measure the similarity between these tasks, DistBO uses a distributional kernel, which learns the similarity between source and target tasks by a joint Gaussian process model.
At the first iteration of optimization for the target task, DistBO uses LCB (Lower Confidence Bound) as the acquisition function to quickly select good hyperparameters. %\footnote{Note that while our goal is to minimize the objective function, DistBO aims to maximize the objective function.}
DistBO models a joint distribution $p(x,y)$ of each task.
However, in MS-CS, modeling $p_T(x,y)$ is impossible because the labeled data of the target task are unavailable.
Therefore, in our experiments, DistBO models the marginal distribution $p(x)$, not the joint distribution $p(x,y)$, of each task.
This setting is also used in the original paper (Section 5.1 of~\citep{law2019hyperparameter}).
If the covariate shift assumption holds, it is sufficient to model $p(x)$, as $p(y|x)$ does not change across tasks.

\subsection{Experimental Setting}
For fair comparison, we used Gaussian Process Upper Confidence Bound (GP-UCB)~\citep{srinivas2010gaussian} as a hyperparameter optimization algorithm for all methods.
In our experiments, the confidence parameter in GP-UCB is set to $2.0$ (following the setting used in~\citep{nguyen2016budgeted,gonzalez2016batch,bogunovic2018adversarially}).

We set the number of evaluations (i.e., $B$ in Algorithm 1 and 2) to $50$.
At the beginning of optimization, we sample $5$ initial points randomly.
A Mat\'{e}rn $5/2$ kernel is used in the implementation of GP-UCB.
Note that our study is formulated as a minimization as in \Eqref{eq:objective}, we thus use LCB instead of UCB as an acquisition function.
We use {\it densratio\_py}\footnote{\href{https://github.com/hoxo-m/densratio\_py}{https://github.com/hoxo-m/densratio\_py}} to estimate the density ratio by uLSIF~\citep{kanamori2009least}.
All the experiments were conducted on Google Cloud Platform (n1-standard-4) or MacBook Pro (2.2 GHz Intel Core i7, 16 GB).

For LI, which utilizes the gradient descent algorithm to optimize a meta-loss function defined by source tasks, we need to set a learning rate $\eta$ and a number of epochs $E$.
Following~\citep{wistuba2015learning}, we set $\eta = 10^{-3}$ and $E = 10^3$.
To obtain the results of DistBO in our experiments, we use the implementation provided by the authors.\footnote{https://github.com/hcllaw/distBO}
We use a Mat\'{e}rn $5/2$ kernel for a fair comparison with other methods, whereas the original implementation uses a Mat\'{e}rn $3/2$ kernel as default.

\begin{figure*}[t]
    \centering
    \hspace*{\fill}
    \subfloat[][Comparing all methods \label{subfig:synthetic}]{
        \includegraphics[width=0.49\linewidth]{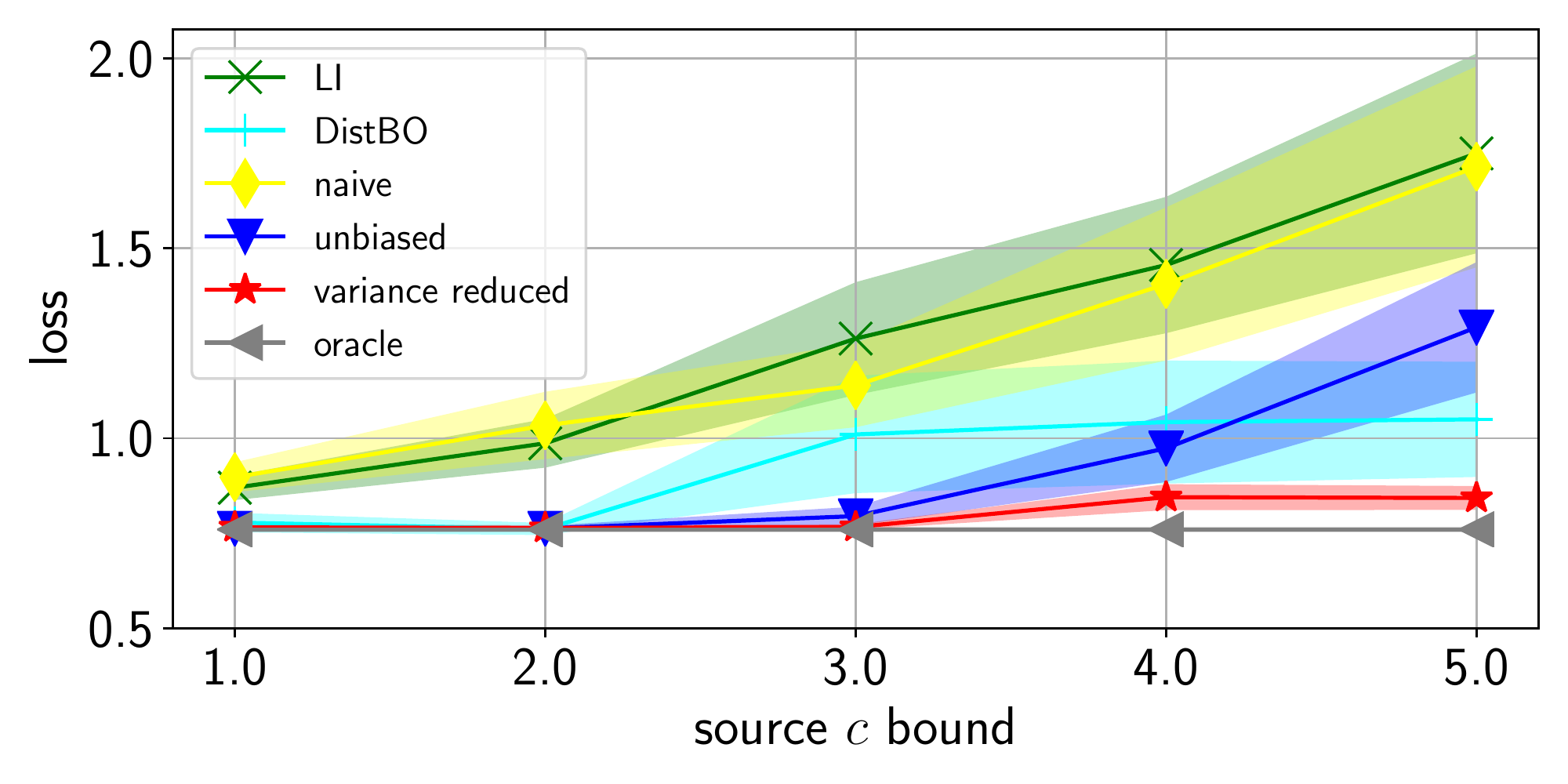}
    }
    \hspace*{\fill}
    \subfloat[][Comparing unbiased and variance reduced \label{subfig:synthetic_unbiased_vr_ratio}]{
        \includegraphics[width=0.49\linewidth]{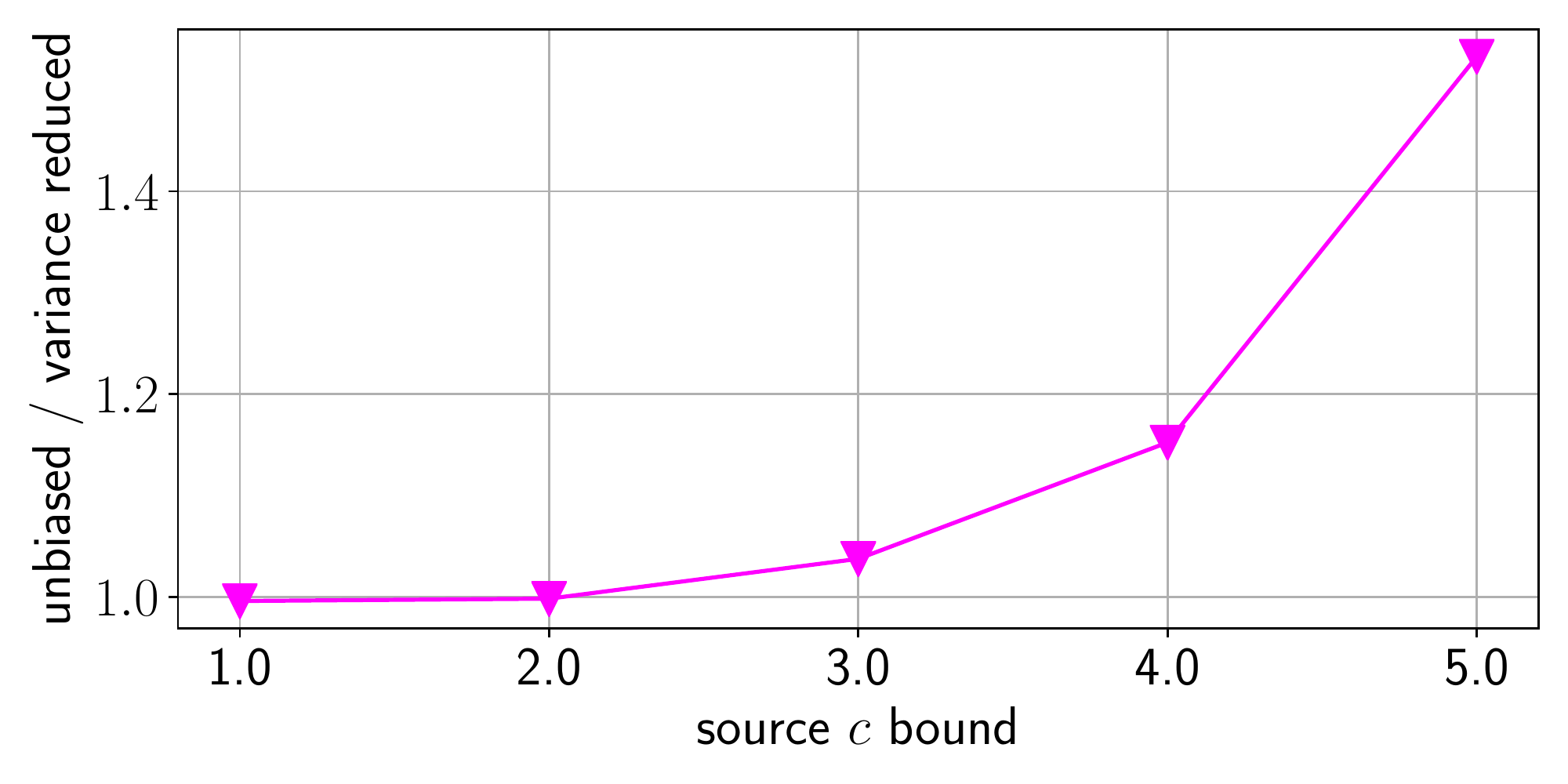}
    }
\hspace*{\fill}
\caption{Results of the experiment on synthetic toy problems over $30$ runs. The horizontal axis represents the prior parameters of the source tasks $c_j^S \in \{ 1.0, \cdots, 5.0 \},$ $\ (j \in \{ 1, \cdots, k \})$. (a) The vertical axis represents the mean and standard error of the performance of each method. (b) The vertical axis represents the ratio of the mean performance of the unbiased estimator to the variance reduced estimators.
} 
\label{fig:toy}
\end{figure*}

\subsection{Toy Problem} \label{sec:toy}
We consider a $1$-dimensional regression problem which aligns with the MS-CS setting.
The toy dataset is generated as follows:
\begin{align*}
    \mu^j &\sim \mathcal{U}(-c_j, c_j), \, \\
    \{ x^j_i \}_{i=1}^{n} \mid \mu^j &\overset{i.i.d.}{\sim} \mathcal{N}(\mu^{j}, 1),  \, \\
    \{ y^j_i \}_{i=1}^{n} \mid \{ x^j_i \}_{i=1}^{n} &\overset{i.i.d.}{\sim} \{ \mathcal{N}(0.7 x^j_i + 0.3, 1) \}_{i=1}^{n}, \nonumber
\end{align*}
where $ \mathcal{U} $ is the uniform distribution, $\mathcal{N}$ denotes the normal distribution, and $c_j \in \mathbb{R}$ is an experimental setting which characterizes the marginal input distribution ($p(x)$) of task $j$.
The objective function $f$ is given by:
\begin{align}
    f(\theta; \mathcal{D}_j) = \frac{1}{n} \sum_{i=1}^n L(\theta, y_i),\ L(\theta, y_i)= (\theta - y_i)^2 / 2.
\end{align}
Similar to the toy experiment in~\citep{law2019hyperparameter}, $\theta \in [-8, 8]$ is a hypothetical hyperparameter, which we aim to optimize.

As we described in Section~\ref{sec:method}, the variance reduced estimator is expected to outperform the unbiased estimator, when the source task and the target task differ significantly.
To demonstrate this, we use various values as $c_j$ to define source tasks ($c_j^S \in \{ 1.0, 2.0, \cdots, 5.0 \}, j \in \{1, \cdots, k\}$) and target task ($c^T = 1.0$).
The source and target distributions are similar when $c_j^S = 1.0 (= c^T)$; in contrast, they are quite different when $c_j^S = 5.0$.
Finally, we set $k = 2$ and $n = 1000$.

Figure~\ref{fig:toy} shows the results of the experiment on the toy problem over $30$ runs with different random seeds.
First, Figure~\ref{fig:toy} (a) indicates that the unbiased and variance reduced estimators significantly outperform the naive method and LI in all settings. 
This is because these estimators unbiasedly approximate the target task objective by addressing the distributional shift, while the naive method and LI do not.
% Moreover, this figure shows the advantage of unbiasedness is highlighted when the distributions of the target and source tasks diverge greatly (i.e., when $c_i^S$ is large.).
% DistBO shows relatively good performance despite the lack of unbiasedness
The variance reduced estimator also outperforms DistBO when there is a source task that is greatly different from the target (when $c$ is large).

Next, we compare the performance of the unbiased and variance reduced estimators in Figure~\ref{fig:toy} (b).
This reports the performance of the unbiased estimator relative to the variance reduced estimator with varying values of $c$.
The result indicates that the advantage of the variance reduced estimator over the unbiased estimator is further strengthened when there is a large divergence between the target and source task distributions.
This result is consistent with our theoretical analysis, which suggests that the variance reduced estimator has the minimum variance.
Finally, as shown in Figure~\ref{fig:toy} (a), the variance reduced estimator achieves almost the same performance as the upper bound (oracle) without using the labels of the target task, suggesting its powerful HPO performance under the MS-CS setting.

\subsection{Real-World Data} \label{sec:hpo_real}

\paragraph{\textbf{Datasets:} }
We use \href{http://archive.ics.uci.edu/ml/datasets/Parkinsons+Telemonitoring}{Parkinson's telemonitoring} (Parkinson)~\citep{tsanas2009accurate} and Graft-versus-host disease (GvHD) datasets~\citep{brinkman2007high} to evaluate our proposed method on real-world problems.

Parkinson consists of voice measurements of $42$ patients with early-stage Parkinson disease collected by using a telemonitoring device in remote symptom progression monitoring.
Each patient has about $150$ recordings characterized by a feature vector with $17$ dimensions.
The goal is to predict the Parkinson disease symptom score of each recording.

GvHD is an important medical problem in the allogeneic blood transplantation field~\citep{brinkman2007high}.
The issue occurs in allogeneic hematopoietic stem cell transplant recipients when donor immune cells in the graft recognize the recipient as foreign and initiate an attack on several tissues. 
The GvHD dataset contains weekly peripheral blood samples obtained from 31 patients characterized by a feature vector with $7$ dimensions.
Following~\citep{muandet2013domain}, we omit one patient who has insufficient data, and subsample the data of each patient to have 1000 data points each.
The goal is to classify \text{CD3+CD4+CD8+} cells, which have a high correlation with the development of the disease~\citep{brinkman2007high}.

\begin{table*}[h]
    \centering
    \caption{Hyperparameter space for SVM on Parkinson's telemonitoring dataset.}
     \label{tab:hp_svm}
    \setlength\tabcolsep{2mm}
    \begin{tabularx}{\textwidth}{p{15em}p{4em}p{4em}p{10em}}
    \toprule
         Hyperparameters & Type & Scale & Search Space \\ 
         \midrule
         Kernel Coefficient             & float & log & $[5.0 \times 10^{-5}, 5.0 \times 10^{3}]$\\
         L2 Regularization Parameter    & float & log & $[5.0 \times 10^{-5}, 5.0 \times 10^{3}]$\\
         \bottomrule
    \end{tabularx}
\end{table*}

\begin{table*}[h]
    \centering
    \caption{Hyperparameter space for LightGBM on GvHD dataset.}
     \label{tab:hp_lgbm}
    \setlength\tabcolsep{2mm}
    \begin{tabularx}{\textwidth}{p{15em}p{4em}p{4em}p{10em}}
    \toprule
         Hyperparameters & Type & Scale & Search Space \\ 
         \midrule
         Max Depth for Tree             & int & linear & $[2, 6]$\\
         Feature Fraction               & float & linear & $[0.1, 1.0]$\\
         Learning Rate                  & float & log & $[1.0 \times 10^{-3}, 1.0 \times 10^{-1}]$\\
         L2 Regularization Parameter    & float & log & $[5.0 \times 10^{-5}, 5.0 \times 10^{3}]$\\
         \bottomrule
    \end{tabularx}
\end{table*}

\paragraph{\textbf{Experimental Procedure:} }
To create the MS-CS setting, we treat each patient as a task for both datasets.
We select one patient as a target task and regard the remaining patients as multiple source tasks.
Then, we employ the following experimental procedure:
\begin{enumerate}
    \item Tune hyperparameters of \textbf{an ML model} by \textbf{an HPO method} using the unlabeled target task and labeled source tasks,
    \item Split the original target task data into $70\%$ training set and $30\%$ test set,
    \item Train \textbf{an ML model} tuned by \textbf{an HPO method} for MS-CS using the training set of the target task, 
    \item Predict target variables (the symptom score for Parkinson and \text{CD3+CD4+CD8+} cells for GvHD) on the test set of the target patient, 
    \item Calculate \textbf{target task objective} (prediction accuracy) and regard it as the performance of the HPO method,
    \item Repeat the above steps 10 times with different seeds and report the mean and standard error over the simulations. 
\end{enumerate} 

Note that \textbf{an HPO method} can be one of LI, DistBO, Naive, Unbiased, Variance Reduced, or Oracle. 
\textbf{An ML model} is either of support vector machine (SVM) or LightGBM as described below.

\paragraph{\textbf{Settings:} }
As for an ML model and a target task objective, we use SVM implemented in \textit{scikit-learn}~\citep{pedregosa2011scikit} and \textit{mean absolute error} (MAE) for Parkinson. 
In contrast, we use LightGBM~\citep{ke2017lightgbm} as an ML model and \textit{binary cross-entropy} (BCE) as a target task objective for GvHD.

Tables~\ref{tab:hp_svm} and~\ref{tab:hp_lgbm} describe the hyperparameter search spaces ($\Theta$) for SVM (Parkinson) and LightGBM (GvHD).
We treat integer-valued hyperparameters as a continuous variable and discretize them when defining an ML model.
For SVM, we use Radial Basis Function kernel in {\it scikit-learn}~\citep{pedregosa2011scikit}.
We use {\it microsoft/LightGBM}\footnote{\href{https://github.com/microsoft/LightGBM}{https://github.com/microsoft/LightGBM}} to implement LightGBM.
For the details of these hyperparameters, please refer to the documentations of scikit-learn\footnote{\href{https://scikit-learn.org/stable/modules/generated/sklearn.svm.SVR.html}{https://scikit-learn.org/stable/modules/generated/sklearn.svm.SVR.html}} and LightGBM\footnote{\href{https://lightgbm.readthedocs.io/en/latest/Parameters.html}{https://lightgbm.readthedocs.io/en/latest/Parameters.html}}.

We normalize the feature vectors of the GvHD dataset as a preprocessing procedure.
In contrast, in the Parkinson dataset, we do not apply standardization or normalization, because these operations cause serious performance degradation of SVM.
In the Parkinson dataset, we select a task (patient) having the maximum number of data as the target task. 
In contrast, the tasks in the GvHD dataset all have the same number of data, thus we select the task (patient) with the first task index as the target.
When we use unbiased and variance reduced estimators, we use $30\%$ of the training set to estimate the density ratio and the remaining $70\%$ to learn the ML models as suggested in Section~\ref{sec:method}.

\begin{table}[ht]
% \small
\centering
\caption{Comparing different MS-CS methods (\std{Mean}{StdErr}). The \best{red} fonts represent the best performance among estimators that are feasible in MS-CS. The mean and standard error (StdErr) are induced by running 10 simulations with different random seeds.}
% \vskip 0.05in
\def\arraystretch{1.15}
\resizebox{\linewidth}{!}{
\begin{tabular}{lcccc}
\toprule
\textbf{Estimators} & & \multicolumn{1}{c}{\textbf{Parkinson (MAE)}} && \multicolumn{1}{c}{\textbf{GvHD (BCE)}}\\
\midrule \midrule
LI && \std{0.41507}{0.1669} && \std{0.19695}{0.0468} \\
DistBO && \std{1.54202}{0.1006} && \std{0.33015}{0.0600} \\
Naive && \std{1.10334}{0.0908}  && \std{0.02121}{0.0052} \\
Unbiased && \std{1.08283}{0.1981} && \std{0.02141}{0.0052} \\
Variance reduced (ours) && \best{\std{0.40455}{0.1755}} && \best{\std{0.01791}{0.0039}} \\ \midrule
Oracle (reference) && \std{0.06862}{0.0011} && \std{0.01584}{0.0043} \\ \bottomrule
\end{tabular}}
% }
% \vskip 0.05in
% \begin{minipage}{\columnwidth} % choose width suitably
% {\small {\it Note}: 
% The \best{red} fonts represent the best performance among estimators using only the unlabeled target task and labeled source task datasets.
% The mean and standard error (StdErr) are induced by running 10 simulations with different random seeds.
% \par}
% \end{minipage}
\label{parkinson}
\end{table}

\paragraph{\textbf{Results:} }
Table~\ref{parkinson} presents the results of the experiments over 10 runs with different random seeds.
In contrast to the results on synthetic data, the performance of DistBO deteriorates significantly.
While DistBO requires a reasonable number of hyperparameter evaluations per source task, our setting allows only a very small number of evaluations per source task, which may lead to learning inaccurate surrogate models for DistBO.
The unbiased estimator performs almost the same with naive on the Parkinson dataset given their standard errors.
Moreover, it slightly underperforms the naive in GvHD, although the unbiased estimator satisfies unbiasedness.
This is because the number of data for each task is small, and the variance issue of the unbiased estimator is highlighted in these data.
This result implies that pursuing only unbiasedness in the approximation of the target task objective is not sufficient in MS-CS.
On the other hand, the proposed variance reduced estimator alleviates the instability issue of the unbiased estimator and performs the best in both datasets.
The results also suggest that the variance reduced estimator works well on both regression (Parkinson) and classification (GvHD) problems.
Therefore, we conclude from its variance optimality and empirical performance that the variance reduced estimator is the best choice to conduct efficient HPO under MS-CS.

% \paragraph{Discussion: }\nomura{We might want to remove this section.}
% The major difference between the proposed method and the baseline methods, LI and DistBO, lies in the method of evaluating hyperparameters.
% While the proposed method evaluates one hyperparameter using all source tasks, LI and DistBO consider the situation where hyperparameters are evaluated for any one source task.
% The effect originated by this difference will become significant when the number of source tasks is large.
% For example, let us consider a scenario where we aim to conduct HPO with DistBO. Here, the available evaluation budget is $B = 100$.
% If the number of source tasks is $2$, we can optimize each source task with evaluation budget $B = 50$ for each source task.
% On the other hand, if the number of source tasks is $50$, the evaluation budget for each source task is only $B = 2$, which makes optimization nearly impossible.
% Actually, in the Appendix C.5 in~\citep{law2019hyperparameter}, the existence of $1230$ hyperparameter evaluations on source tasks is assumed on Parkinson's dataset to lead to a reasonable performance of DistBO.
% In contrast, our method works properly even if there is a limited evaluation budget for the source tasks.
% In our experiment on Parkinson's dataset (in Section \ref{sec:hpo_real}), we assume that only $50$ hyperparameter evaluations on source tasks are available.

\section{Conclusion}
In this work, we explore a novel HPO problem under MS-CS with the goal of enabling efficient HPO with only an unlabeled target task and \textit{multiple} labeled source task datasets. Towards that end, we propose the variance reduced estimator and show that it achieves the \textit{variance optimality} by leveraging the \textit{task divergence measure}. Moreover, we show that the resulting HPO procedure is \textbfit{no-regret} even under the difficult MS-CS setting.
Empirical evaluation demonstrated that the proposed HPO procedure helps us identify useful hyperparameters without the labels of the target task.

In assumption~\ref{assumption2}, we assumed that the conditional outcome distribution remains the same across all tasks. 
Although it is a standard assumption in the covariate shift literature, it might not hold in uncertain real-world scenarios. 
Exploring the effective HPO procedure without this assumption, for example, by leveraging a few labeled samples from the target task, is an interesting future research direction.
Another limitation of our study is that the scales of the HPO experiments are small. 
This is due to the fact that there are limited real-world datasets that enable the experiments under MS-CS.
We argue that empirically verifying our proposed HPO procedure on larger datasets is necessary to make our procedure more convincing and reliable.

\bibliographystyle{ACM-Reference-Format}
\bibliography{main}

\end{document}

%% file: 999_Notations.tex
\usepackage[utf8]{inputenc} % allow utf-8 input
\usepackage[T1]{fontenc}    % use 8-bit T1 fonts
\usepackage{hyperref}       % hyperlinks
\usepackage{url}            % simple URL typesetting
\usepackage{booktabs}       % professional-quality tables
\usepackage{amsfonts}       % blackboard math symbols
\usepackage{nicefrac}       % compact symbols for 1/2, etc.
\usepackage{microtype}      % microtypography

\usepackage{amsmath,amsthm}

\DeclareMathOperator*{\argmin}{arg~min}
\newcommand{\textbfit}[1]{\textbf{\textit{#1}}}

\newcommand{\mE}{\mathbb{E}}
\newcommand{\mV}{\mathbb{V}}
\newcommand{\mP}{\mathbb{P}}

\newcommand{\tf}{f_{T} (\theta)}
\newcommand{\calD}{\mathcal{D}}
\newcommand{\ubf}{\hat{f}_{UB} \bigl(\theta; \{ \D_{S^j} \}_{j=1}^{k} \bigr)}

\newcommand{\vrf}{\hat{f}_{VR} \bigl(\theta; \{ \D_{S^j} \}_{j=1}^{k} \bigr)}

\newcommand{\lamf}{\hat{f}_{\boldsymbol{\lambda}} \bigl(\theta; \{ \D_{S^j} \}_{j=1}^{k} \bigr)}

\newcommand{\D}{\mathcal{D}}
\newcommand{\jsim}{Div \left( T \, || \, S^j \right)}
\newcommand{\ijx}{x_i^j}
\newcommand{\ijy}{y_i^j}

\newtheorem{theorem}{{\em Theorem}}
\newtheorem{definition}{{\em Definiton}}
\newtheorem{lemma}[theorem]{{\em Lemma}}
\newtheorem{assumption}{{\em Assumption}}
\newtheorem{proposition}[theorem]{{\em Proposition}}
\newcommand{\Eqref}[1]{Eq. (\ref{#1})}

\newcommand{\Asmref}[1]{Assumption \ref{#1}}
\newcommand{\Thmref}[1]{Theorem \ref{#1}}

\newcommand{\Tabref}[1]{Table~\ref{#1}}

\usepackage{algorithm}
\usepackage{algpseudocode}

\usepackage{graphicx}
\usepackage[caption=false]{subfig}

\usepackage{tabularx}

\usepackage{color}
\definecolor{dkgreen}{rgb}{0,0.6,0}
\definecolor{customgray}{rgb}{0.25,0.25,0.25}
\definecolor{customred}{rgb}{0.8,0.05,0.05}
\definecolor{customblue}{rgb}{0.05,0.05,0.8}
\newcommand{\best}[1]{\textcolor{customred}{\textbf{#1}}}
\newcommand{\std}[2]{#1 \textcolor{customgray}{\small{$\pm$#2}}}

\usepackage{color}